\documentclass{article}

\usepackage{microtype}
\usepackage{graphicx}
\usepackage{subfigure}
\usepackage{booktabs} 
\usepackage{amsmath,amssymb,amsthm}
\usepackage{hyperref}

\newtheorem{theorem}{Theorem}
\newtheorem{lemma}{Lemma}
\theoremstyle{remark}
\newtheorem*{rem}{Remark}

\usepackage[accepted]{whi2018}

\icmltitlerunning{Why Interpretability in Machine Learning?}

\begin{document}

\twocolumn[
\icmltitle{Why Interpretability in Machine Learning? \\An Answer Using Distributed Detection and Data Fusion Theory}



\icmlsetsymbol{equal}{*}

\begin{icmlauthorlist}
\icmlauthor{Kush R.\ Varshney}{ibm}
\icmlauthor{Prashant Khanduri}{syr}
\icmlauthor{Pranay Sharma}{syr}
\icmlauthor{Shan Zhang}{syr}
\icmlauthor{Pramod K.\ Varshney}{syr}
\end{icmlauthorlist}

\icmlaffiliation{ibm}{IBM Research, Yorktown Heights, New York, USA}
\icmlaffiliation{syr}{Syracuse University, Syracuse, New York, USA}

\icmlcorrespondingauthor{K.\ R.\ Varshney}{krvarshn@us.ibm.com}

\icmlkeywords{Interpretability, Distributed Detection}

\vskip 0.3in
]



\printAffiliationsAndNotice{}  

\begin{abstract}
As artificial intelligence is increasingly affecting all parts of society and life, there is growing recognition that human interpretability of machine learning models is important.  It is often argued that accuracy or other similar generalization performance metrics must be sacrificed in order to gain interpretability.  Such arguments, however, fail to acknowledge that the overall decision-making system is composed of two entities: the learned model and a human who fuses together model outputs with his or her own information.  As such, the relevant performance criteria should be for the entire system, not just for the machine learning component. In this work, we characterize the performance of such two-node tandem data fusion systems using the theory of distributed detection.  In doing so, we work in the population setting and model interpretable learned models as multi-level quantizers.  We prove that under our abstraction, the overall system of a human with an interpretable classifier outperforms one with a black box classifier.  
\end{abstract}

\section{Introduction}
\label{sec:intro}

``When you create a Human+AI team, the hard part isn't the `AI'.  It isn't even the `Human'.  It's the `+''' \citep{Case2018}.

\citet{Nirenburg2017} dichotomizes artificial intelligence (AI) systems into cognitive prostheses, ones intended to replace humans, and cognitive orthotics, ones intended to enhance human performance on tasks.  Also known as intelligence augmentation, orthotic systems are intended to collaborate with humans, and as such, must be proficient both at the task at hand and at communicating with humans.  Computer systems can communicate at rates on the order of billions of bits per second, but humans can only do so on the order of hundreds of bits per second \cite{Lawrence2018}.  A strength of humans, however, is intuition and reasoning \citep{Case2018}.  Thus, to consider an AI system successful as an augmentation for decision support, it must bring forth relevant information for the decision making task, but must also communicate at an appropriate rate and in a way that allows a human recipient of the information to tap his or her strengths of intuition and reasoning.

Arguments in recent debates have claimed that it is only the accuracy of machine learning models that matters, not their interpretability.  However, taking this view ignores the fact that the overall system in high-stakes settings is a machine learning model communicating with a human who makes the final decision, and thus it is the accuracy of the overall system that is of relevance.  Interpretable machine learning models are an appropriate means for communication between AI and human \cite{DhurandharILS2017}; the contribution of this paper is to abstractly model the overall system and theoretically show the system performance advantage of interpretable machine learning models over black box machine learning models.

In this paper, we consider the population setting (the limit as the number of samples goes to infinity, allowing access to the probability distributions of the data) and appeal to the theory of distributed detection and data fusion \cite{Varshney1997}.  We take this approach because it represents the simplest setting to understand the phenomenon without being too simple.  Examples of working in the population setting abound in the machine learning literature \cite{GrettonBRSS2006,RavikumarLLW2007,ScottBH2013,ShenderL2013,MenonW2018}.  We also restrict ourselves to binary classification for simplicity, but there is nothing fundamentally different if we consider multicategory classification.  This work should be differentiated from recent contributions that discuss hybrids of interpretable and black box models \cite{Wang2018}, because here, we are concerned with the hybrid of \emph{humans} and models.

The specific way we model the classification system is as a two-node sensor network in a tandem architecture.  The first node is the machine learning model that makes a local observation, puts it through the Bayes optimal decision rule (i.e.\ computes the likelihood ratio statistic),\footnote{It is not obvious a priori that a local Bayes decision followed by quantization is optimal in this decision making architecture, but must be proved \cite{Varshney1997,zhu2013data}.} and transmits a quantized version of this statistic to the second node, the human.  The human has an independent local observation which it fuses with the information received from the model node to produce the final decision.  The quantizer restricted to two quantization levels is used to model a black box model that can only transmit its classification.  A quantizer with more than two quantization levels is used to model an interpretable classifier; one can imagine decision trees, rule sets, local post hoc explanations, and other human interpretable model forms \cite{MalioutovVED2017} as partitions of the decision space similar to the effect of quantization of the likelihood ratio.

We prove that the Chernoff information between the two likelihood functions (class-conditional probabilities) participating in the final human decision is greater for systems with more quantization levels.  Via the Chernoff theorem, this implies that the Bayes performance of the system with more quantization levels is better.  That is: interpretable models perform better than black box models.

Note that we do not intend to imply that more levels yields greater interpretability, but only that three or more levels is an interpretable regime. In reality, a very large number of quantization levels stops being a good model for an interpretable machine learning model because humans have limits to how much information they can process. Therefore, let us assume that we are not in the regime with a large number of levels.  In addition, we note that the proposed stylized abstraction of interpretability does not differentiate between simply quantizing the output score of a black box classifier with probabilistic outputs (which is still uninterpretable) and a truly human interpretable classifier; a more extensive formulation is needed to capture this distinction and incorporate additional aspects of interpretability such as the ability to examine feature-specific errors and vagaries caused by dataset shift.  Other limitations of this work are discussed in Section \ref{sec:conclusion}.

\section{Problem Setup}
\label{sec:setup}

Consider the binary classification problem in the population setting with two nodes collaborating via a tandem network illustrated in Figure \ref{fig:system}.  
\begin{figure}
\vskip 0.2in
\begin{center}
\centerline{\includegraphics[width=\columnwidth]{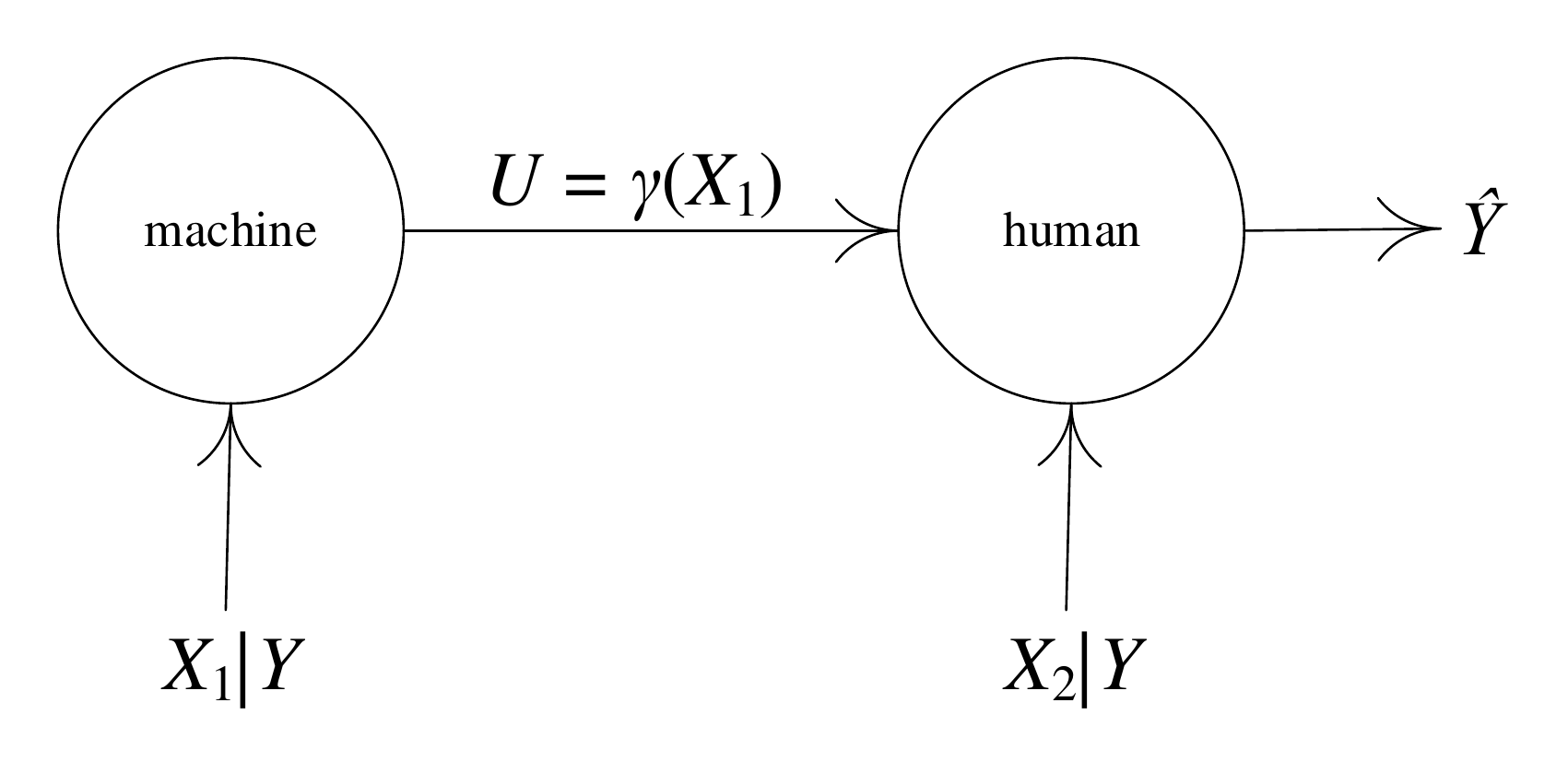}}
\caption{System model.}
\label{fig:system}
\end{center}
\vskip -0.2in
\end{figure}
Let the features observed by the two nodes, $X_1$ and $X_2$, be conditionally independent given the class label $Y \in \{0,1\}$ and governed by the likelihoods $f_{X_1\mid Y}(x_1\mid Y=y)$ and $f_{X_2\mid Y}(x_2\mid Y=y)$.  Sensor 1, a model for the machine learning model, transmits $U = \gamma(X_1)$ to sensor 2, a model for the human, where $\gamma(\cdot)$ is the composition of two functions, the likelihood ratio: 
\begin{displaymath}
\Lambda(X_1) = \frac{f_{X_1\mid Y}(x_1\mid y=1)}{f_{X_1\mid Y}(x_1\mid y=0)},
\end{displaymath}
and an optimal quantizer.  Sensor 2 acts as a fusion center and bases its classification on both $U$ and $X_2$.  This classification rule $\hat{y}(U,X_2)$ is the globally Bayes optimal likelihood ratio test that thresholds 
\begin{displaymath}
\Lambda(U,X_2) = \frac{f_{U,X_2\mid Y}(u,x_2\mid y=1)}{f_{U,X_2\mid Y}(u,x_2\mid y=0)}.
\end{displaymath}

The quantizer has $k \ge 2$ levels.  The case $k=2$ models a black box and the case $k>2$ models an interpretable model. Specifically,
\begin{equation}
U =
  \begin{cases}
    1, & \Lambda(X_1) < b_1,\\
    2, & b_1 \leq \Lambda(X_1) < b_2,\\
    \vdots & \quad \quad \vdots\\
    k,  & b_{k-1} \leq \Lambda(X_1)\\
  \end{cases}
\end{equation} 
where $\{b_1, b_2, \ldots, b_{k-1}\}$ are the quantization thresholds.

\section{Performance Characterization}
\label{sec:perf}

Our aim is to now show that the system having more quantization levels, i.e.\ larger $k$, has better classification performance.  In service of that goal, we first provide a relevant inequality and prove that having more quantization levels leads to larger Chernoff information (or Chernoff divergence) \cite{Chernoff1952} between the likelihood functions. Then we explicate how this relationship between Chernoff informations yields the conclusion of systems with interpretable classifiers performing better than systems with black box classifiers.
\begin{lemma}
\label{lem:geom}
The following inequality is satisfied by posynomial functions $f$ for $\lambda\in (0,1)$:
\begin{multline}
f \left( p_1^{1-\lambda}q_1^{\lambda},\ldots,p_n^{1-\lambda}q_n^{\lambda} \right) \\\leq f \left( p_1,\ldots,p_n \right)^{1-\lambda}  f \left( q_1,\ldots,q_n \right)^{\lambda}
\end{multline}
with equality if and only if $p_i = q_i$, $i=1,\ldots,n$.
\end{lemma}
\begin{proof}
This is Eq.\ 8 in \citet{BoydKVH2007}.
\end{proof}
\begin{rem}
This geometric convexity inequality is a generalization of the arithmetic mean--geometric mean inequality.
\end{rem}
\begin{theorem}
\label{thm:chernoffinformation}
Consider two learnable tandem networks as described above with different numbers of quantizer levels $k$ and $k'$ with $k' > k$ and corresponding quantized transmissions $U$ and $U'$.  Then, the following relationship among Chernoff informations holds:
\begin{multline}
C\left(f_{U',X_2 \mid Y}(u', x_2 \mid y=1) \| f_{U',X_2 \mid Y}(u', x_2 \mid y=0)\right) \\> C\left (f_{U,X_2 \mid Y}(u, x_2 \mid y=1) \| f_{U,X_2 \mid Y}(u, x_2 \mid y=0)\right).
\end{multline}
\end{theorem}
\begin{proof} Since $X_1$ and $X_2$ are conditionally independent, for $k$-level quantization, we have:
\begin{multline}
\label{eq:C1}
C\left(f_{U,X_2 \mid Y}(u, x_2 \mid y=1) \| f_{U,X_2 \mid Y}(u, x_2 \mid y=0)\right) \\= C\left(f_{U \mid Y}(u \mid y=1) \| f_{U \mid Y}(u \mid y=0)\right) \\+ C\left(f_{X_2 \mid Y}(x_2 \mid y=1) \| f_{X_2 \mid Y}(x_2 \mid y=0)\right).
\end{multline}

The second term in \eqref{eq:C1} does not depend on the quantization levels, so we focus only on the first term involving $U$. Recalling that $U$ is a discrete random variable taking values $\{1,\ldots,k\}$, this first term is given by
\begin{multline}
\label{eq:C2}
C\left(f_{U \mid Y}(u \mid y=1) \| f_{U \mid Y}(u \mid y=0)\right) \\= -\log\min_{\lambda\in(0,1)} \sum_{j= 1}^k p_j^{1-\lambda} q_j^{\lambda},
\end{multline}
where $p_j = P(u=j \mid y=1)$ and $q_j = P(u=j \mid y= 0)$.  Similarly, for $k'$-level quantization, we have:
\begin{multline}
\label{eq:C3}
C\left(f_{U' \mid Y}(u' \mid y=1) \| f_{U' \mid Y}(u' \mid y=0)\right) \\= -\log\min_{\lambda\in(0,1)} \sum_{i= 1}^{k'} {p'_i}^{1-\lambda} {q'_i}^{\lambda},
\end{multline}
where $p'_i = P(u'=i \mid y=1)$ and $q'_i = P(u'=i \mid y= 0)$.

Without loss of generality, assume that the quantizer is a uniform quantizer.  Then,
\begin{align*}
q_j &= P \left( \Lambda(X_1) \in \left[ \tfrac{j-1}{k}, \tfrac{j}{k} \right) \mid y = 0 \right), \\
p_j &= P \left( \Lambda(X_1) \in \left[ \tfrac{j-1}{k}, \tfrac{j}{k} \right) \mid y = 1 \right),
\end{align*}
for $j  = 1,\ldots, k$, and
\begin{align*}
q'_i &= P \left( \Lambda(X_1) \in \left[ \tfrac{i-1}{k'}, \tfrac{i}{k'} \right) \mid y = 0 \right), \\
p'_i &= P \left( \Lambda(X_1) \in \left[ \tfrac{i-1}{k'}, \tfrac{i}{k'} \right) \mid y = 1 \right),
\end{align*}
where $i = 1,\ldots, k'$.

For $k$-level quantization, an interval $\left[ \frac{r-1}{k}, \frac{r}{k} \right)$ contains $\rho = k'/k$ intervals of length $1/k'$.\footnote{For the sake of simplicity, we assume $k'$ to be an integer multiple of $k$. The proof holds if that is not the case but requires additional bookkeeping.} Therefore, we have
\begin{equation}
\label{eq:union}
\left[ \tfrac{r-1}{k}, \tfrac{r}{k} \right) = \bigcup_{i=1}^{\rho} \left[ \tfrac{\rho (r-1) +i-1}{k'}, \tfrac{\rho (r-1) +i}{k'} \right).
\end{equation}

Then, using (a) the definition of $q_j$, (b) equation \eqref{eq:union}, (c) the disjoint property of intervals, and (d) the definition of $q'_i$:
\begin{align}
\label{eq:qrho}
q_r &\overset{(a)}{=} P\left(\Lambda(X_1) \in  \left[ \tfrac{r-1}{k}, \tfrac{r}{k} \right) \mid y = 0\right)\nonumber\\
&\overset{(b)}{=} P \left(\Lambda(X_1) \in  \bigcup_{i=1}^{\rho} \left[ \tfrac{\rho (r-1) + i - 1}{k'}, \tfrac{\rho (r-1)  +i}{k'} \right)  \mid y = 0\right) \nonumber\\
&\overset{(c)}{=} \sum_{i=1}^{\rho} P \left(\Lambda(X_1) \in  \left[ \tfrac{\rho (r-1) +i-1}{k'}, \tfrac{\rho (r-1)  +i}{k'} \right) \mid y = 0\right) \nonumber\\
&\overset{(d)}{=} \sum_{i=1}^{\rho} q'_{\rho (r-1) +i}.
\end{align}
Similarly,
\begin{equation}
\label{eq:prho}
p_r = \sum_{i=1}^{\rho} p'_{\rho (r-1)  +i}.
\end{equation}

Then, using (a) equation \eqref{eq:C2}, (b) equations \eqref{eq:qrho} and \eqref{eq:prho}, (c) the geometric convexity inequality of Lemma \ref{lem:geom}, (d) collecting all the $p'_i$ by summing $r$ over $1,\ldots,k$ and $i$ over $1,\ldots,\rho$, and (e) equation \eqref{eq:C3}:
\begin{align*}
&C\left(f_{U \mid Y}(u \mid y=1) \| f_{U \mid Y}(u \mid y=0)\right)\\ &\overset{(a)}{=} -\log\min_{\lambda\in(0,1)}\sum_{r= 1}^k p_r^{1-\lambda} q_r^{\lambda} \\
&\overset{(b)}{=} -\log\min_{\lambda\in(0,1)} \sum_{r=1}^{k} \left( \sum_{i=1}^{\rho} p'_{\rho (r-1) +i} \right)^{1-\lambda} {\left( \sum_{i=1}^{\rho} q'_{\rho (r-1) +i} \right)}^{\lambda} \\
&\overset{(c)}{<} -\log\min_{\lambda\in(0,1)} \sum_{r=1}^{k} \sum_{i=1}^{\rho} {p'_{\rho (r-1) +i}}^{1-\lambda} {q'_{\rho (r-1) +i}}^{\lambda} \\
&\overset{(d)}{=} -\log\min_{\lambda\in(0,1)} \sum_{i=1}^{k'} {p'_i}^{1-\lambda}{q'_i}^{\lambda} \\
&\overset{(e)}{=} C\left(f_{U' \mid Y}(u' \mid y=1) \| f_{U' \mid Y}(u' \mid y=0)\right).
\end{align*}
Step (c) is a strict inequality because the $p'_i$ and the $q'_i$ are different when the classification task is learnable.  The result follows by reintroducing the second terms of equation \eqref{eq:C1}.
\end{proof}
\begin{theorem}
\label{thm:chernoff}
The best achievable exponent in the Bayesian probability of error in a binary classification problem with class labels $Y$ and features $X$ is $C\left(f_{X \mid Y}(x \mid y=1) \| f_{X \mid Y}(x \mid y=0)\right)$.
\end{theorem}
\begin{proof}
Known as the Chernoff Theorem, this is Theorem 11.9.1 in \citet{CoverT2006}.
\end{proof}
\begin{theorem}
\label{thm:result}
The probability of error in the tandem classification network described above with $k = 2$ quantizer levels is larger than the network with $k' > 2$ quantizer levels.
\end{theorem}
\begin{proof}
This is a direct consequence of Theorem \ref{thm:chernoffinformation} and Theorem \ref{thm:chernoff}.
\end{proof}
\begin{rem}
This analysis makes no assumption about the relative quality of observations $X_1$ and $X_2$ made by the machine learning model and the human respectively.  It continues to hold even if the two are very differently distributed (even on different variables), and the human features $X_2$ are very noisy\hspace{1pt}---\hspace{1pt}possibly relating to some intuition that is difficult to pin down and represent as data.
\end{rem}

\begin{rem}
This analysis is for the Bayesian detection setting, which is the standard for supervised classification in machine learning.  The detection theory literature is often oriented towards the Neyman--Pearson setting, which does occasionally also arise in machine learning \cite{RigolletT2011}.  The current analysis can be repeated for the Neyman--Pearson paradigm with only minor changes: switching Chernoff information to Kullback--Leibler divergence, switching Lemma \ref{lem:geom} to the log-sum inequality, and switching Theorem \ref{thm:chernoff} to the Chernoff-Stein Lemma (Theorem 11.8.3 in \citet{CoverT2006}).
\end{rem}

\section{Related Work in Distributed Detection and Estimation}

The motivation for this study is to develop an understanding of the human--machine decision-making team in the presence of interpretable models, but it also provides a new contribution to the distributed detection and data fusion literature.  Although two-node tandem sensor networks with quantization have been studied before, the analysis conducted in Section~\ref{sec:perf} has not been done.  

\citet{zhu2013data} investigated the problem of sufficiency-based data reduction in tandem fusion systems with quantization.  They showed that quantizing the sufficient statistics achieves the same optimal inference performance as quantizing the raw observations.  Their results applied to systems with conditionally independent observations and also to conditionally dependent observations under certain conditions.  It is because of this result that we can equivalently quantize either $\Lambda(X_1)$ or $X_1$ in this work.  This paper did not, however, characterize the difference in inference performance for different numbers of quantization levels as we do here.

Several works study the problem of whether the noisier sensor or the less noisy sensor in a two-sensor tandem fusion system should optimally perform the fusion and take the final decision \cite{akofor2013optimal,zhu2013interactive, akofor2013interactive,song2007some, song2009performance}.  In many settings, it is the less noisy sensor that should optimally make the decision, but this is not universally true.\footnote{Scholars have raised this same question in discussing the collaboration of humans and AI in decision making.  In this context, Kahneman recently stated \cite{MITIDE2018}, ``You can combine humans and machines, provided the machine has the last word!''}

Finally, \citet{ShenVZ2012} and \citet{ShenVZ2014} study a problem similar to ours with continuous $Y$, i.e.\ regression or estimation, and thus have characterizations hinging on Fisher information rather than Chernoff information as in the analysis herein.  This work, like all of the others cited in this section, does not relate the analysis to interpretable machine learning and human decision making.

\section{Limitations and Conclusion}
\label{sec:conclusion}

In this paper, we have modeled the overall decision-making procedure involving humans and AI systems as one involving quantized communication from the AI to the human who makes the final decision.  For analysis purposes, we have considered the population setting in which we can examine the probability distributions involved, thereby avoiding the complexities in analyzing the finite data sample regime.  We have shown that interpretable AI (taken to be systems with more than two quantization levels) yields lower probability of error than black box AI (taken to be systems with two quantization levels).  

One limitation of this work is that we have assumed that the two nodes, the human and the machine, have features that are independent conditioned on the label.  However, it is quite reasonable to imagine that the two feature sets would be correlated, perhaps even strongly so and even statistically dependent.  It is our conjecture that we can analyze the conditionally dependent case using \citet{zhu2013data} as a starting point and following the approach that allows \citet{ShenVZ2012} (conditionally independent measurements) to be extended to \citet{ShenVZ2014} (conditionally dependent measurements), and that the main conclusion would not change.

Another limitation comes from working in the population setting.  As discussed earlier, when we are in this setting, an optimal (scalar) quantizer of the sufficient statistic and an optimal (vector) quantizer of the raw features yield equivalent performance.  The sufficient statistic represents a perfect Bayes classifier and the optimal quantizer of the raw features also somehow captures perfect Bayes classification. Therefore, it is as if black boxes, post hoc interpretations, and directly interpretable models all have equivalent accuracies which is not necessarily true with finite training data.  To extend the current analysis to the finite sample case, \citet{NguyenWJ2005} and \citet{PreddKP2006} may prove instructive; although they are for more general topologies than two-node tandems, which tend to introduce many simplifications and allow for analysis that may not otherwise be possible.


\bibliography{distdetectinterp}
\bibliographystyle{icml2018}

\end{document}